\documentclass[letterpaper, 10 pt, conference]{ieeeconf}  

\IEEEoverridecommandlockouts                              

\let\labelindent\relax

\usepackage{amsthm}
\usepackage{enumitem}

\usepackage{epsfig} 
\usepackage{mathptmx} 
\usepackage{amsmath} 
\usepackage{amssymb}  

\usepackage{amsfonts}

\usepackage{xcolor}

\usepackage{graphicx} 
\usepackage{algorithm}
\usepackage{algorithmicx}
\usepackage[noend]{algpseudocode}

\usepackage{booktabs}
\usepackage{cite}
\usepackage{epstopdf}

\newtheorem{lemma}{Lemma}
\newtheorem{theorem}{Theorem}
\newtheorem{definition}{Definition}
\renewcommand{\QED}{\QEDopen}

\usepackage{multirow}

\title{\LARGE \bf
Assisted Shortest Path Planning for a Convoy through a Repairable Network
}

\author{Abhay Singh Bhadoriya$^{1}$, Christopher Montez$^{1}$, Sivakumar Rathinam$^{1}$,\\ Swaroop Darbha$^{1}$, David W. Casbeer$^{2}$, and Satyanarayana G. Manyam$^{3}$%
\thanks{$^{1}$With Texas A \& M University, College Station, TX, USA
        {\tt\small abhay.singh@tamu.edu, yduaskme@tamu.edu, dswaroop@tamu.edu, srathinam@tamu.edu}}%
\thanks{$^{2}$David W. Casbeer is with the Controls Center, Air Force Research Laboratory, WPAFB, OH, USA {\tt\small david.casbeer@us.af.mil}}%
\thanks{$^{3}$Satyanarayana G. Manyam is with Infoscitex corporation, a DCS Company, Dayton, OH, USA 
        {\tt\small smanyam@infoscitex.com}}%
\thanks{This work was supported in part by AFOSR LRIR No. 21RQCOR084.}
\thanks{Distribution Statement A. Approved for public release, distribution unlimited. Case Number: AFRL-2022-0779.}
}

\begin{document}

\maketitle

\begin{abstract}

    In this article, we consider a multi-agent path planning problem in a partially impeded environment. The impeded environment is represented by a graph with select road segments (edges) in disrepair impeding vehicular movement in the road network. A convoy wishes to travel from a starting location to a destination while minimizing some accumulated cost. The convoy may traverse an impeded edge for an additional cost (associated with repairing the edge) than if it were unimpeded. A second vehicle, referred to as a service vehicle, is simultaneously deployed with the convoy. The service vehicle assists the convoy by repairing an edge, reducing the cost for the convoy to traverse that edge. The convoy is permitted to wait at any vertex to allow the service vehicle to complete repairing an edge. The service vehicle is permitted to terminate its path at any vertex. The goal is then to find a pair of paths so the convoy reaches its destination while minimizing the total time (cost) the two vehicles are active, including any time the convoy waits. We refer to this problem as the Assisted Shortest Path Problem (ASPP). We present a generalized permanent labeling algorithm to find an optimal solution for the ASPP. We also introduce additional modifications to the labeling algorithm to significantly improve the computation time and refer to the modified labeling algorithm as $GPLA^*$. Computational results are presented to illustrate the effectiveness of $GPLA^*$ in solving the ASPP. We then give concluding remarks and briefly discuss potential variants of the ASPP for future work.

\end{abstract}

\section{INTRODUCTION}

There has been an increasing interest in the study of cooperative behavior between multiple autonomous agents operating in a shared environment to complete some given task \cite{wagner2015subdimensional,sharon2015conflict,de2013push,standley2010finding,erdmann1987multiple,silver2005cooperative}. This growing interest is due to the many applications which can make use of the cooperation of multiple agents, such as search-and-rescue \cite{li2016hybrid}, cooperative manipulation \cite{parker_thesis}, foraging \cite{steels1990}, surveillance\cite{lynne1999}, and safe escort using coordinated UAV and UGV systems \cite{li2016hybrid,minaeian2015vision,garcia2017coordinated,chung2021offensive}. The manner in which the agents cooperate will vary depending on the collective goal considered. The agents may employ a divide-and-conquer approach and complete independent tasks, such as in a search-and-rescue application. The agents may instead be required to directly assist each other, such as in the case of cooperative manipulation where two or more agents work together to physically move an object in space (box-pushing \cite{donald_box_pushing}). A combination of these two approaches may also be considered, as in this paper. The use of multiple agents to complete a given task can potentially lead to performance improvements (faster times, more tasks done per unit cost, etc.). The use of multiple agents also allows for a wider variety of situations to be addressed. We refer the reader to \cite{cooperative_Robots_Survey} for a comprehensive survey of cooperative behavior of multiple autonomous agents. 

The focus of this paper is the path planning of two autonomous agents operating in a partially impeded environment. As an example, this impeded environment may be a warehouse in which agents are organizing goods. The pathways in the warehouse may have some obstructions, such as fallen goods, that hinder the agents' ability to maneuver. This impeded environment is represented by a graph where select edges represent obstructions using edge weights that will be defined later. These edges are assumed to be known \textit{a priori} and are referred to as impeded edges. In the general case, obstructions may be physical or abstract. A designated agent, referred to as the convoy vehicle for the remainder of this paper, must leave a specified starting location and reach a given destination while minimizing some accumulated cost. The convoy will be assumed to have the capability to handle the obstructions present, but doing so will incur some cost in addition to the cost of travel. A second designated agent, referred to as the service vehicle for the remainder of this paper, can be simultaneously deployed to assist the convoy by clearing these obstructions itself as the convoy is traveling in the environment. As the service vehicle travels, it also incurs some cost. When the service vehicle clears an obstruction, it incurs some additional cost. Once an obstruction is cleared, it is assumed to remain cleared. The convoy may choose to wait at a vertex for some additional cost to allow the service vehicle to clear obstructions. The service vehicle may terminate at any vertex in the graph without continuing to incur additional cost as the convoy continues to its destination. The objective is then to find a pair of paths for the convoy and service vehicle such that the total cost accumulated for the two vehicles is minimized. For ease of discussion, we will refer to this problem as the Assisted Shortest Path Problem (ASPP). We provide a mathematical formulation for the ASPP in Section \ref{Problem_statement}.

Restricted variations of the ASPP have been studied in \cite{montez2021finding} and \cite{montez2021approximation}. In these papers, the convoy was not able to clear any obstructions by itself. This is the same as the cost of the convoy traversing any impeded edge being infinite. In \cite{montez2021finding}, a mixed-integer linear programming formulation was presented.
An approximation algorithm was presented in \cite{montez2021approximation} for the case where the convoy may share an impeded edge with the service vehicle. The addition of the convoy's ability to traverse through impeded edges in this paper adds a significant level of complexity and allows for modeling more complex scenarios.

The main contribution of the paper is the development of a generalized, permanent labeling algorithm ($GPLA$) that finds an optimal solution for the ASPP. Permanent labeling algorithms were introduced in \cite{desrochers1988generalized} for the path planning of a {\it single} vehicle in a graph while adhering to various path structural constraints. In this approach \cite{irnich2005shortest}, the state of the vehicle is stored as a label, and the label is updated as the vehicle moves to reflect the changes in the state. The key part of this approach lies in defining suitable labels and their updates so that partial solutions that can potentially lead to an optimal solution are not discarded. In this paper, we develop this approach to handle multiple, cooperative vehicles and solve the ASPP. After we develop the $GPLA$, we introduce filters to further reduce the search space and improve the performance of the algorithm. We refer to this refined $GPLA$ with filters as $GPLA^*$. The procedure we present can be further generalized to handle more vehicles as needed.

The structure of this paper is as follows. Section \ref{Problem_statement} presents the mathematical formulation for the ASPP. Section \ref{Solution_strategy} presents the $GPLA$ and $GPLA^*$. Section \ref{Results} presents computational results for $GPLA^*$ and the $GPLA$. $GPLA*$ is shown to be significantly better than the $GPLA$. Several families of problem instances are constructed to show the influence of various aspects of the ASPP on the runtime of $GPLA^*$ and the structure of optimal solutions to the ASPP. Section \ref{Conclusion} presents concluding remarks and future work.

\color{black}

\section{Problem Statement} \label{Problem_statement}

Let $G = (V, E)$ be an undirected, connected graph representing an impeded environment. $V$ is the set of vertices and $E$ is the set of undirected edges. The convoy and service vehicle start at vertices $p$ and $q$, respectively, where $p$ and $q$ need not be distinct. Let $d \in V$ be the destination of the convoy. The service vehicle may terminate at any vertex at any time and therefore has no destination. For the remainder of this paper, we will use cost and time interchangeably. Changes can be easily made to consider more general costs. Edges $K \subseteq E$ needing repair are known \textit{a priori} and are referred to as impeded edges. The remaining edges are called unimpeded edges. We say an impeded edge is serviced if either the convoy or service vehicle has completely traversed that edge. In this paper, we use repair and service for an impeded edge interchangeably to indicate removing an obstruction (physical or abstract) in the impeded environment.
We use the term impeded travel cost to refer to the cost for a vehicle to take an impeded edge that has not yet been serviced and use the term unimpeded travel cost otherwise. Each edge $e \in E$ has four edge weights of the form $( T^u_e, T^i_e, \tau^u_e, \tau^i_e)$. $T^u_e$ and $T^i_e$ are the unimpeded and impeded travel cost for the convoy, respectively. The unimpeded and impeded travel costs for the service vehicle are $\tau^u_e$ and $\tau^i_e$, respectively.
All costs are non-negative and are known. For unimpeded edges, $T^i_e = T^u_e$ and $\tau^i_e = \tau^u_e$.
For impeded edges, we have $T^i_e > T^u_e$ and $\tau^i_e > \tau^u_e$. Therefore, the cost of the convoy taking edge $e$ is reduced if it is first serviced by the service vehicle. All edges are undirected so the cost in both directions are taken to be identical. When an impeded edge $e$ is serviced, it remains serviced.

We assume the following: (1) the convoy and service vehicles can share vertices and edges without conflict, (2) the two vehicles start at the same time, (3) the two vehicles  communicate at all times and information is shared in a negligible amount of time, and (4) the service vehicle always travels faster than the convoy at any edge. The convoy is allowed to wait at any vertex to give the service vehicle time to repair impeded edges. Waiting will incur some additional cost. In this paper, this waiting cost is simply the time the convoy waits. We also allow the service vehicle to wait, but doing so will incur zero additional cost. This was done for two main reasons. First, depending on the application the service vehicle may be able to temporarily shut off or idle for a negligible amount of power consumption. Second, since the service vehicle is assumed to travel faster than the convoy and all costs are non-negative, it is clear an optimal solution will never have the service vehicle wait at a vertex and so we can set the wait cost to be zero for mathematical convenience. A more general waiting cost can be included, but the subsequent definitions and algorithms will need to be modified.

\begin{figure}[!ht]
\begin{center}
\includegraphics[width=\linewidth]{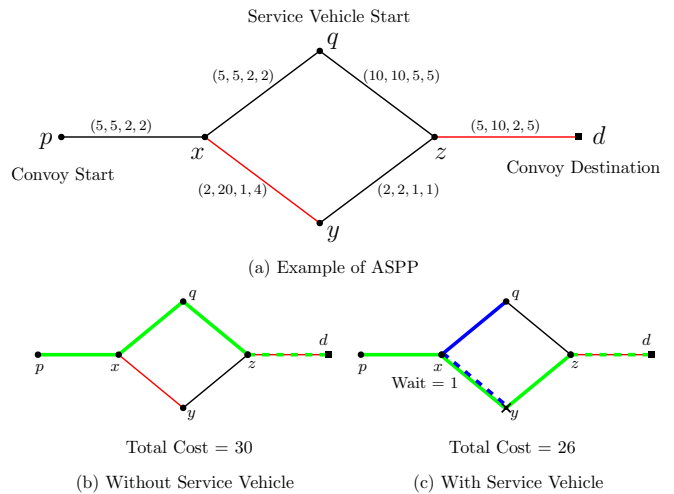}
\caption{(a) A sample instance for the ASPP. Black edges represent unimpeded edges and red edges represent impeded edges. The edge weights associated with each edge are of the form $(T^u_e,T^i_e,\tau^u_e,\tau^i_e)$. In (b) and (c), the convoy and the service vehicle paths are represented by green and blue colors, respectively. The dashed path represents the impeded edge is serviced by the vehicle associated with the corresponding color.
In (b), the optimal solution for the convoy without the assistance of the service vehicle is shown. In (c), the optimal solution for the convoy with the assistance of the service vehicle is shown; in this case, the convoy will wait at node $x$ for $1$ unit of time.
The service vehicle terminates its path at node $y$, which is represented by a cross. It is clear the solution in (c) is better than the solution in (b).}
\label{ASPPsample} 
\end{center}
\end{figure}

Let $X_{pa}$ be a path from $p$ to $a \in V$ for the convoy and $\bar{X}_{qb}$ be a path from $q$ to $b \in V$ for the service vehicle. We note that a vehicle's path will also potentially have waiting at one or more vertices. The two paths are coupled using the rules previously outlined. The cost of $X_{pa}$ will depend on the decisions made by the service vehicle in path $\bar{X}_{qb}$ and vice-versa. {We note that the service vehicle may also choose to remain at $q$.} The accumulated cost of these two paths is the sum of the time the convoy reaches $a$ and the time the service vehicle first reaches $b$ (it may choose to terminate at $b$) while adhering to the motion rules and travel cost rules. This cost includes any waiting time by the convoy. In other words, the accumulated cost is the total cost of the vehicles being active. If the service vehicle never initially leaves $q$, the cost of $\bar{X}_{qq}$ is zero. We denote the total cost by $C(X_{pa}, \bar{X}_{qb}$). The ASPP is then to find two coupled paths $X_{pd}$ and $\bar{X}_{qv}$, where $v$ is any vertex in $V$, that minimizes $C(X_{pd}, \bar{X}_{qv})$.

A sample problem for the ASPP is shown in Fig. \ref{ASPPsample}(a). 
Fig. \ref{ASPPsample}(b) represents the least-cost path for the convoy, $X_{pd} = (p,\,x,\,q,\,z,\,d)$ with total cost $30$, without the help of the service vehicle. The convoy will repair the edge $(z,d)$ in this solution.
Fig. \ref{ASPPsample}(c) represents the optimal solution with the help of the service vehicle. The service vehicle will take the path $\bar{X}_{qy} = (q,\,x,\,y)$, repair the edge $(x,y)$ and then terminate at $y$. The convoy will start from $p$ and wait at $x$ for $1$ unit of time to let the service vehicle finish the repair and traverse the serviced edge $(x,y)$. The convoy will then repair the edge $(z,d)$ to reach the destination. The total cost for this solution is $26$.

\section{Generalized Permanent Labeling Algorithms} \label{Solution_strategy}

We first define an abstract object called a label to store relevant information on decisions made by both vehicles. Recall the definition of a path and the accumulated cost from Section \ref{Problem_statement}.
\begin{definition}[Label] \label{label}
    Suppose the convoy and service vehicle have taken paths $X_{pi}$ and $\bar{X}_{qj}$ to $i$ and $j$ in $G$ in times $T_i$ and $\tau_j$, including waiting, while accumulating a total cost $C_{ij}$. For each edge $e$ serviced by either vehicle, create a tuple $(e, \,t_e)$ where $t_e$ is the time edge $e$ was serviced. The set $V_{ij}$ is the collection of the tuples $(e, \, t_e), \, \forall e \in K$. We then define a label $\lambda_{ij}$ as
    \begin{equation*}
        \lambda_{ij} = (i, j, T_i, \tau_j, C_{ij}, V_{ij}).
    \end{equation*}
\end{definition}
\noindent For a label $\lambda_{ij}$, we define $S(V_{ij})$ to be the set of serviced edges in $V_{ij}$. 

\color{black}






We next introduce resource extension functions (REFs) associated with each label. REFs describe how a label $\lambda_{ij}$ is extended to create a new label $\lambda_{lm}$. Extending a label corresponds to feasible extensions of the paths $X_{pi}$ and $\bar{X}_{qj}$ from decisions involving edges $(i, l)$ and $(j, m)$, respectively. We allow for $i = l$ or $j = m$, but not simultaneously\footnote{Both $i = l$ and $j = m$ is not allowed because this condition implies both the vehicles are just waiting without performing any motion or servicing task. This clearly does not lead to an optimal solution.}. 


\begin{definition}[Resource Extension Functions] \label{REF}
Suppose paths \color{black} $X_{pi}$ and $\bar{X}_{qj}$ associated with label $\lambda_{ij}$ are extended by edges $e_x = (i,l)$ and $e_y = (j, m)$, respectively, to create a new label $\lambda_{lm}$. Let $I_{ij} = K \setminus S(V_{ij})$ be the set of remaining impeded edges. \color{black} Then, 






\begin{align*}
    T_l &= 
    \begin{cases}
    T_i + T_{e_x}^u, \hspace{4.25cm} e_x \in E \setminus K \\
    T_i + \min (T_{e_x}^i, T_{e_x}^u + \max(0,t_{e_x}-T_i)), \quad e_x \in S(V_{ij}) \\
    T_i + T_{e_x}^i, \hspace{4.25cm} e_x \in I_{ij} \\
    \max(T_i, \tau_m),  \hspace{3.3cm} i = l \, \, \wedge \, \, j\neq m
    \end{cases}\\
    \tau_m &=
    \begin{cases}
    \tau_j + \tau_{e_y}^i , \quad \quad \quad \quad e_y \in I_{ij} \\
    \tau_j +\tau_{e_y}^u, \quad \quad \quad \quad e_y \in E \setminus I_{ij} \\ 
    \max(\tau_j,T_l), \hspace{0.9cm} j=m \, \, \wedge \, \, i\neq l 
    \end{cases} \\
    C_{lm} &= 
    \begin{cases}
    C_{ij} + (T_l - T_i) + (\tau_m - \tau_j), \quad j\neq m \\
    C_{ij} + (T_l - T_i), \hspace{2cm} j = m
    \end{cases} \\
    V_{lm} &= \begin{cases}
        V_{ij},  \hspace{2.75cm} e_x, e_y \in E \setminus I_{ij} \\
        V_{ij} \cup \{(e_x, T_l)\}, \quad \quad \quad  \,e_x \in I_{ij} \, \, \wedge \, \, e_y \in E \setminus I_{ij} \\
        V_{ij} \cup \{(e_y, \tau_m)\}, \quad \quad \quad e_x \in E \setminus I_{ij} \, \, \wedge \, \, e_y \in I_{ij} \\
        V_{ij} \cup \{(e_x, T_l), (e_y, \tau_m)\}, \quad e_x, e_y \in I_{ij}
    \end{cases}
\end{align*}



\end{definition}

\noindent The REFs in Definition \ref{REF} encode the cost rules previously outlined in Section \ref{Problem_statement} and tracks the edges that have been serviced and the times they were serviced. Each label is associated with a feasible trajectory pair, including the optimal solution.
%
There are an infinite number of possible labels due to cycles and waiting. We introduce a dominance rule to reduce the number of labels under consideration to a finite number.
\color{black}
\begin{definition}[Dominance Rule] \label{domi}
Consider two labels $\lambda_{ij}'$ and $\lambda_{ij}$. We say $\lambda_{ij}'$ dominates $\lambda_{ij}$ if and only if

\begin{enumerate}
    \item $T_i' \leq T_i$ 
    \item $\tau_j' \leq \tau_j$ 
    \item $S({V'_{ij}}) \supseteq S(V_{ij}) $ 
    \item $t_e' \leq t_e \quad \forall e\in S(V_{ij})$
\end{enumerate}
If $\lambda_{ij}' = \lambda_{ij}$, one label can be discarded arbitrarily if only one optimal solution is required. If all optimal solutions are needed, then keep both labels. 
\end{definition}
\noindent Conditions 1 and 2 state that both vehicles have reached the same pair of vertices in less time. Note that we do not require a separate dominance condition for cost as it follows from Conditions 1 and 2. If a more general cost is used, an additional condition $C'_{ij} \leq C_{ij}$ must be included. Condition 3 states at least all serviced edges in $\lambda_{ij}$ have also been serviced in $\lambda'_{ij}$. Condition 4 states each serviced edge in $\lambda_{ij}$ was serviced earlier or at the same time in $\lambda'_{ij}$. 
\color{black}

\begin{theorem} \label{nondominated theorem}
The extensions of only the non-dominated labels need to be considered to obtain the optimal solution.
\end{theorem}

\begin{proof}
Let $\lambda'_{ij}$ and $\lambda_{ij}$ be two distinct \color{black} labels with $\lambda'_{ij}$ dominating $\lambda_{ij}$. We need to show any feasible extension of $\lambda_{ij}$ will be dominated by the same extension of $\lambda'_{ij}$. \color{black}

Both $\lambda'_{ij}$ and $\lambda_{ij}$ must have the same feasible extensions.
Let the labels be extended by $e_x=(i,l)$ and $e_y =(j,m)$ for the convoy and service vehicle, respectively, resulting in new labels $\lambda_{lm}$ and $\lambda'_{lm}$\color{black}. We allow for $i = l$ and $j = m$, but not simultaneously.
Using Definitions \ref{REF} and \ref{domi}, we observe the following: \color{black}

\begin{itemize}
    \item If $i \neq l$ and $e_x \notin S(V_{ij}) \cup S(V'_{ij})$, then we have $T'_l \leq T_l$.
    
    \item Suppose $i \neq l$ and $e_x \in S(V_{ij}) \cap S(V'_{ij})$. From $t'_{e_x} \leq t_{e_x}$,
    \begin{align*}
                T_i' + \min (T_{e_x}^i, T_{e_x}^u + &\max(0,t'_{e_x}-T'_i)) \leq\\ T_i + \min (&T_{e_x}^i, T_{e_x}^u + \max(0,t_{e_x}-T_i)),
    \end{align*}
    which implies $T'_l \leq T_l$.
    
     \item Suppose $i \neq l$ and $e_x \in S(V'_{ij}) \setminus S(V_{ij})$. Then
    \begin{align*}
            T_i' + \min (T_{e_x}^i, T_{e_x}^u + \max&(0,t'_{e_x}-T'_i)) \leq T_i + T_{e_x}^i,
    \end{align*}
    which implies $T'_l \leq T_l$.
    
    \item Suppose $j \neq m$. Since $S(V'_{ij}) \supseteq S(V_{ij})$ and $\tau'_j \leq \tau_j$, from Definition \ref{REF} we must have $\tau'_m \leq \tau_m$.
    
    \item Suppose $i = l$ and $j \neq m$. Then
    \begin{equation*}
        \max(T'_i, \tau'_m) \leq \max(T_i, \tau_m)
    \end{equation*}
    which implies $T'_l \leq T_l$.
    
    \item Suppose $j = m$ and $i \neq l$. Then
    \begin{equation*}
        \max(\tau'_j, T'_l) \leq \max(\tau_j, T_l),
    \end{equation*}
    which implies $\tau'_m \leq \tau_m$.
    
    \item The same extension is used so clearly $S(V'_{ij}) \supseteq S(V_{ij})$ implies $S(V'_{lm}) \supseteq S(V_{lm})$.
    
    \item Suppose $e_x \in K$.
    
    \begin{itemize}
        \item If $e_x \notin S(V_{ij}) \cup S(V'_{ij})$, from $T'_l \leq T_l$ it follows that $t'_{e_x} \leq t_{e_x}$.
        
        \item If $e_x \in S(V'_{ij}) \setminus S(V_{ij})$, then $t'_{e_x} \leq T'_i \leq T'_l \leq T_l$ and so $t'_{e_x} \leq t_{e_x}$.
    \end{itemize}
    
    \item Suppose $e_y \in K$. 
    
    \begin{itemize}
        \item If $e_y \notin S(V_{ij}) \cup S(V'_{ij})$, from $\tau'_j \leq \tau_j$ it follows that $t'_{e_y} \leq t_{e_y}$. 
        
        \item If $e_y \in S(V'_{ij}) \setminus S(V_{ij})$, then $t'_{e_y} \leq \tau'_j \leq \tau'_m \leq \tau_m$ and so $t'_{e_y} \leq t_{e_y}$.
    
    \end{itemize}

    \color{black}
\end{itemize}

We see $\lambda'_{lm}$ will always dominate $\lambda_{lm}$.
Since the optimal solution must be a non-dominated label, we need to only consider non-dominated labels. 

\end{proof}
From Theorem \ref{nondominated theorem}, we can solve the ASPP by repeatedly extending non-dominated labels, including the label corresponding to the initial configuration, until there are no more unique non-dominated labels that can be generated. Then, the non-dominated label that has the convoy position at $d$ and with the lowest cost corresponds to an optimal solution. 

We also make a few trivial observations for optimal solutions to the ASPP that will reduce the number of non-dominated labels to consider:
\begin{itemize}
    \item[(i)] The service vehicle will never begin to move again after pausing at a vertex.
    
    \item[(ii)] The service vehicle, if deployed, will only terminate immediately after servicing all of its assigned impeded edges.
    
    \item[(iii)] The convoy will only wait at an end of an impeded edge it uses later.
    
    \item[(iv)] A convoy waiting to use an impeded edge will wait until the service vehicle has serviced that edge.
    
    \item[(v)] The convoy will never wait at a vertex after the service vehicle has terminated its motion.
\end{itemize}

\noindent These observations are a direct consequence of the structure of the cost rules from Section \ref{Problem_statement}. It should be noted these observations may not hold for more general cost structures and variants of the ASPP involving additional constraints. Using these observations, we can filter some extensions of non-dominated labels that will never eventually result in an optimal solution. To do this, we introduce two Boolean variables $SV_{term}$ and $\delta$ for each label and write
\begin{equation*}
    \lambda_{ij} = (i, j, T_i, \tau_j, C_{ij}, V_{ij}, SV_{term}, \delta).
\end{equation*}
$SV_{term}$ is true when the service vehicle has terminated and is false otherwise. The label $\lambda_{pq}$ associated with the initial state of the vehicles has $SV_{term}$ set to false by default. When extending a label $\lambda_{ij}$ with $SV_{term}$ currently set to false, we create (1) all feasible extensions $\lambda_{lm}$ where $j \neq m$ and $SV_{term}$ set to false and (2) all feasible extensions $\lambda_{lj}$ with $SV_{term}$ set to true. From observation (i), we can enforce $SV_{term}$ remain true for all extensions of a label with $SV_{term}$ set to true. The variable $\delta$ is initially set to false and is defined in the following manner: Consider a label $\lambda_{ij}$ where the convoy is at vertex $i$ that is an end of an impeded edge $(i, l) \notin S(V_{ij})$. For this impeded edge, the extension of $\lambda_{ij}$ will create (1) a label $\lambda_{im}$ with $\delta$ set to true and (2) a label $\lambda_{lm}$ with $\delta$ set to false. Next, consider a label $\lambda_{ij}$ with $\delta$ set to true. When extending $\lambda_{ij}$, we create (1) all labels $\lambda_{lm}$ where $(i, l) \in S(V_{ij})$ and (2) all labels $\lambda_{im}$ with $\delta$ set to true. When extending a label results in the convoy moving to a new vertex, $\delta$ is set to false. From this definition, $\delta$ captures the cooperative behavior between the service vehicle and the convoy when interacting with impeded edges. The rules for updating $SV_{term}$ and $\delta$ can be directly included with the REFs in Definition \ref{REF}.

The generalized permanent labeling algorithm ($GPLA$) then works as follows. We create an initial label $\lambda_{pq} = (p, \, q, \, 0, \, 0, \, 0, \, \emptyset$, \text{false}, \text{false}).  We initialize two lists $L_{open}$ and $D$ with $\lambda_{pq}$. We then take $\lambda_{pq}$ out of $L_{open}$ to find all its feasible extensions using the REFs of Definition \ref{REF} and the additional extension rules for $SV_{term}$ and $\delta$ previously described. The non-dominated extensions are then added to $L_{open}$ and $D$. We continue the process of using a chosen selection rule (last-in-first-out or LIFO in our implementation) to remove a label $\lambda_{ij}$ from $L_{open}$, find its extensions, and store the resulting non-dominated labels in $L_{open}$ and $D$. If a label $\lambda_{ij}$ with $i = d$ is removed from $L_{open}$, we do not extend it. This processes is repeated until $L_{open}$ no longer has any labels to extend and the algorithm terminates. Once the algorithm has terminated, the label $\lambda_{dj}$ in $D$ with minimum cost corresponds to an optimal solution to the ASPP. The trajectories for the optimal solution can be found by iteratively going through the predecessors of the label. Therefore, we require some way to store the predecessor label for any given label $\lambda_{ij}$. We have done this by directly including the parent label in the definition of $\lambda_{ij}$ in our implementation.

\color{black}

\begin{algorithm}
        \caption{Generalized Permanent Labeling Algorithm - $A^*$ $(GPLA^*)$ \label{alg:PLA^*}}
        \begin{algorithmic}[1]
            \State \textbf{Input:}  $G(V,E), p, q, d, h, UB$ 
            \State $\lambda_{start}\gets$ Initialization($p,q$)
            \State $D \gets \{\lambda_{start}\} $ \Comment{non-dominated labels} 
            \State $L_{open} \gets \{ \lambda_{start} \} $  \Comment{open list (sorted heap)}
            
            \While{$L_{open} \neq \emptyset $}
                \State $\lambda \gets$ Select$(L_{open})$
                    \If {ConvoyPosition$(\lambda)=d$}
                    \State \Return $\lambda$
                    \EndIf
            
            \For{$n$ $\gets$ Extensions($\lambda$)}   \Comment{$n=(l,m)$}
                \State $\lambda' \gets REF(\lambda,n)$
                    \If{$f(\lambda') \leq UB$ \textbf{and} $\lambda'$ is NonDominated }
                        \State $L_{open} \gets L_{open} \cup \{ \lambda' \}$
                        \State $D \gets D \cup \{ \lambda' \}$ 
                    \EndIf
            \EndFor
            \EndWhile

        \end{algorithmic}
\end{algorithm}

\subsection{Generalized Permanent Labeling Algorithm - $A^*$ ($GPLA^*$)}

We now introduce two final modifications to the permanent labeling algorithm to further reduce the number of labels that need to be extended and the overall computation time as a result. These modifications are motivated by the $A^*$ algorithm \cite{hart1968formal} used for solving the single agent shortest path planning problem. Therefore, we refer to $GPLA$ with these modifications as $GPLA^*$. Algorithm \ref{alg:PLA^*} shows the pseudo-code for $GPLA^*$.

\subsubsection{Early Termination}

In $GPLA$, we repeatedly extend labels from $L_{open}$ until we have exhausted all possible non-dominated labels. In general, the selection rule used to choose labels from $L_{open}$ to be extended can be arbitrary. The algorithm will always converge to the optimal solution irrespective of any selection rule. In permanent labeling algorithms, it is common to select the label most recently added to $L_{open}$ (LIFO) or find the label in $L_{open}$ with the least cost (best-first). The selection rule used can have a significant \color{black} impact on the termination time of the algorithm. 

For a single agent shortest path planning problem, the $A^*$ algorithm uses a heuristic-cost-based extension method to reach the destination while reducing the search space.
We use a similar approach.
For any label $\lambda_{ij}$, we define the heuristic cost $h_i$ to be a lower bound on the cost for the convoy to reach $d$ from $i$. 
In our implementation, we have taken $h_i$ to be the least cost path from $i$ to $d$ while treating all impeded edges as unimpeded, i.e., we set the edge weight to be $T^u_e$ for each edge. We define $f$-cost for $\lambda_{ij}$ to be the sum of $C_{ij}$ and $h_i$. In other words, $f$-cost is a lower bound on the cost for the convoy to reach the destination from the state corresponding to $\lambda_{ij}$.

\begin{lemma} \label{A* lemma}
If a label $\lambda_{ij}$ with the least $f$-cost is selected to be extended at every iteration, then the first such label with $i = d$ will correspond to an optimal solution to the ASPP.
\end{lemma}
\begin{proof}
From Definition \ref{REF}, for any extension from $\lambda_{ij}$ to $\lambda_{lm}$, we have $C_{ij} \leq C_{lm}$. We then see
\begin{align*}
    &h_i \leq h_l + (C_{lm} - C_{ij}) \\
    \Rightarrow &h_i + C_{ij} \leq h_l + C_{lm}
\end{align*}
which implies that the $f$-cost will never decrease with extensions. Let $\lambda_{dj}$ be the first label with the convoy at the destination selected to be extended.
This implies that all the other labels will have $f$-cost greater than or equal to the $f$-cost of $\lambda_{dj}$. Note that $f$-cost of $\lambda_{dj}$ is equal to $C_{dj}$. This implies that no other label can reach the destination with total cost lower than $C_{dj}$.
\end{proof}

\noindent In order to utilize the result in Lemma \ref{A* lemma}, we sort $L_{open}$ by the $f$-cost of each label and use a best-first selection rule to select the label with the smallest $f$-cost (Line 6 of Algorithm \ref{alg:PLA^*}). When a label $\lambda_{ij}$ with $i = d$ is selected, the algorithm terminates and this label is returned as the optimal solution. The heuristic cost $h_i$ at each vertex $i \in V$ is computed before starting the algorithm. This is indicated by the input $h$ in Algorithm \ref{alg:PLA^*}.
\color{black}

\subsubsection{Cost filter}

As previously discussed, when extending a label we only keep the resulting non-dominated labels. Checking for dominance is computationally expensive due to the third and fourth conditions of Definition \ref{domi}. To further reduce the number of dominance checks needed, we compute an upper bound $(UB)$ on the optimal solution to the ASPP before starting the labeling algorithm and use this upper bound to prematurely discard non-dominated labels. The upper bound used is the shortest path for the convoy from $p$ to $d$ without deploying the service vehicle. In other words, we set the cost of each unimpeded edge to be $T^u_e$ and the cost of each impeded edge to be $T^i_e$ and find the shortest path on $G$ with these edge weights. When a label is generated, before checking for dominance we check if the $f$-cost exceeds the upper bound and if so we discard the label.



\section{Results} \label{Results}

All algorithms were implemented in Python 3.6 and the computations were done on an MSI laptop (8 Core Intel i7-7700HQ processor @ 2.80 GHz, 16 GB RAM). For the analysis, we used a grid structure for the graph as it can easily represent a real-world scenario, like a warehouse, and is also easy to reproduce for verification. For all of the instances, the origin and the destination of the convoy, $p$ and $d$, were chosen to be at diagonally opposite ends of the grid to avoid trivial cases. Impeded edges can be chosen randomly or strategically so that the convoy has to traverse through at least one impeded edge to reach the destination. This can be achieved by choosing the impeded edges to make a cut between $p$ and $d$. A cut between $p$ and $d$ for an instance is defined as the set of edges $\{(a,b)\in E \, |\, a\in A, \, b\in B\}$ for any two disjoints sets $A,B\subset V$, such that $A\cup B =V$ and $p\in A$ and $d\in B$. An instance with $c$ cuts is constructed so that the impeded edge set $K$ only consists of edges from $c$ randomly chosen cuts. 

To help illustrate the graph structure used in the subsequent computational studies, Fig. \ref{sampleinstance} provides a sample instance. In Fig. \ref{sampleinstance}, the convoy and service vehicle optimal solution paths are presented on the graph. In this instance, we have set $T_e^u = 10$ and $\tau_e^u = 1$ for all edges and set $T_e^i = 40$ and $\tau_e^i = 6$ for all impeded edges. In the optimal solution, the service vehicle starts from $(7,4)$ and repairs the following sequence of impeded edges: $((3,3),(3,2))$, $((3,1),(2,1))$, $((2,1),(2,0))$ and $((7,3),(8,3))$. The convoy has to wait at $(2,0)$ for $4$ units of time in the optimal solution. The optimal solution cost for this instance is $172$.

The remainder of this section is as follows. We first show a comparison between the run times of $GPLA$ and $GPLA^*$. Next, we analyze $GPLA^*$ over three different classes of instances. In {\it Class 1}, we aim to test the computational limits of $GPLA^*$ as we increase the graph size and the number of impeded edges. In {\it Class 2}, we want to study the significance of unavoidable impeded edges on the optimal solution cost. We accomplish this by increasing the number of cuts on a grid graph of fixed size.
\color{black}
Finally, {\it Class 3} instances are designed to understand the impact of the service vehicle starting position on the optimal solution.

\begin{figure}[]
    \begin{center}
        \includegraphics[width=\linewidth, trim = 0 0 0 -0.1in]{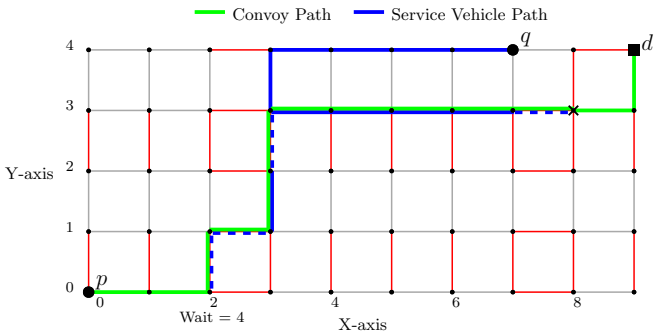}
        \caption{A $5\times10$ grid instance with convoy and service vehicle paths for the optimal solution. Impeded edges are represented by the red color and the remaining edges are unimpeded.
        The service vehicle terminates its path at node $(8,3)$ which is represented by a cross.
        The dashed path represents the impeded edge is serviced by the vehicle associated with the corresponding color. }
        \label{sampleinstance}
    \end{center}
\end{figure}

\subsection{Comparison between $GPLA$ and $GPLA^*$} \label{PLA-baseline}

To compare $GPLA$ and $GPLA^*$, we have created grid instances with randomly chosen impeded edges while enforcing the fraction of impeded edges to be $|K|/|E| = 0.1$ for every instance. The grid sizes were varied with increasing size and can be seen in Table \ref{GPLA-GPLA*}. For each grid size, 50 instances were generated. The starting position of the service vehicle was chosen randomly for each instance.
For the convoy, the unimpeded travel cost, $T_e^u$, was randomly chosen from the range $[10, \,15]$ for all edges and the impeded travel cost, $T^i_e$, was randomly chosen from the range $[40, \, 50]$ for the impeded edges. 
For the service vehicle, the unimpeded travel cost, $\tau^u_e$, was set to be $1$ unit for all edges and the impeded travel cost, $\tau^i_e$, was randomly chosen from the range $[2,\,6]$ for the impeded edges. For unimpeded edges, $T^i_e = T^u_e$ and $\tau^i_e = \tau^u_e$. This cost structure was chosen to encourage collaboration between the vehicles while keeping the decision-making non-trivial.

\begin{table}[!ht]
\vspace*{0.1in}
\caption{Computational time comparison: $GPLA$ vs $GPLA^*$}
\label{GPLA-GPLA*}
    \begin{center}
        \begin{tabular}{c c c c c }
            \toprule
            Grid-Size & $T_{GPLA}$ & $T_{GPLA^*}$  & $O_{GPLA}$ & $O_{GPLA^*}$ \\ \cmidrule(lr){1-5}
            $4\times3$    &  25 ms &  0.6 ms  &  578  &  10 \\
            $4\times4$    &  \,1.1 s   &  0.6 ms  &  8180  &  14 \\
            $4\times5$    &  149 s  &  0.7 ms  &  111804 &  17 \\
            $4\times6$    &  283 s &  0.9 ms  &  189365 &  28 \\
            \bottomrule
        \end{tabular}
    \end{center}
\end{table}

\begin{table*}[]
\vspace*{0.1in}
\caption{Class 1 Results}
\label{class1}
    \begin{center}
        \begin{tabular}{c c c| c c| c c|}
        \cline{2-7}
        \rule{0pt}{8pt}
        & \multicolumn{2}{|c|}{$|K|/|E|$ = 0.3} & \multicolumn{2}{c|}{$|K|/|E|$ = 0.4} &
        \multicolumn{2}{c|}{$|K|/|E|$ = 0.5}\\
        \hline
        \multicolumn{1}{|c|}{Grid Size} & time (sec.)* & $\gamma$ & time (sec.)* & $\gamma$  &  time (sec.)* & $\gamma$ \\ \hline

        \multicolumn{1}{|c|}{6$\times$6} & 0.47\,$\pm$ 2.5 & 1.00 & 7.6\,$\pm$\,27\,\, & 1.00 & 99 $\pm$ 201 & 0.86 \\ 
        
        \multicolumn{1}{|c|}{7$\times$7} & \,\,14\,\,$\pm$ 67 & 1.00 & \,\,58\,\,$\pm$\,162 & 0.88 & 197 $\pm$ 303 & 0.48 \\ 
        
        \multicolumn{1}{|c|}{8$\times$8} & \,\,17\,\,$\pm$ 66 & 0.96 & 105 $\pm$ 205 & 0.70 & 145 $\pm$ 210 & 0.22 \\ 
        
        \multicolumn{1}{|c|}{9$\times$9} & \,\,\,\,43\,\,$\pm$\,144 & 0.86 & 148 $\pm$ 250 & 0.48 & 158 $\pm$ 189 & 0.12 \\ 
        
        \multicolumn{1}{|c|}{10$\times$10} & \,\,\,\,45\,\,$\pm$\,119 & 0.80 & 126 $\pm$ 223 & 0.36 & 128 $\pm$ 129 & 0.08 \\ \hline
        \end{tabular}

        \item[*] time indicates the average computation time $\pm$ standard deviation.  
        
    \end{center}
\end{table*}

Table \ref{GPLA-GPLA*} shows the comparison between the two algorithms and highlights the effectiveness of $GPLA^*$ over $GPLA$. Both algorithms always result in the same optimal solution cost, hence it is not included in Table \ref{GPLA-GPLA*}. Each row represents the average value over $50$ randomly generated instances. $T_{GPLA}$ and $T_{GPLA^*}$ represents the average computational time for $GPLA$ and $GPLA^*$, respectively. $O_{GPLA}$ and $O_{GPLA^*}$ shows the average number of extended labels for both algorithms. The comparison between the computation times and the extended labels clearly indicate the effectiveness of the modifications introduced for $GLPA^*$.

\subsection{Class 1 instances}

To test the computational limits of the $GPLA^*$, we generated instances from $6\times6$ to $10\times10$ grid sizes. For each instance, the impeded edges were chosen randomly and we used the same cost structure as in \ref{PLA-baseline}. The fraction of impeded edges is denoted by $|K|/|E|$.
For each grid size, the fraction of impeded edges were varied to be 0.1, 0.2, 0.3, 0.4, and 0.5. We generated $50$ random instances for each grid size and fraction of impeded edges. The starting position for the service vehicle was chosen randomly for each instance. An instance is said to be \textit{successful} if it terminated with the optimal solution within 900 seconds. The success rate ($\gamma$) is defined as the fraction of all successful instances over the total number of instances. For a majority of the instances with $|K|/|E| = 0.1$ and $0.2$, we observed that the impeded edges did not form a cut and so the convoy had at least one unimpeded path from $p$ to $d$. Therefore, we do not include those results in our discussion.

The results for Class 1 are shown in Table \ref{class1}. In Table \ref{class1}, each entry in the time columns represents the average computation time of successful instances along with the standard deviation. We see the success rate decreases with increasing grid size and increasing fraction of impeded edges. The average computation time and the standard deviation increase with increase in grid size and fraction of impeded edges. For cases with relatively low success rate ($\gamma < 0.5$), the average computation time may not follow the same trend as the samples are skewed. 
From this set of instances, we see the computation time is significantly affected by the number of impeded edges. This is to be expected, as $GPLA^*$ must produce additional labels whenever a label is such that the convoy position is an end of an impeded edge. Similarly, an additional label is also generated whenever an extension is such that the service vehicle traverses an impeded edge, servicing it. As the number of impeded edges increases, the number of additional labels generated will begin to explode. The upper bound cost filter introduced for $GPLA^*$ significantly reduces the number of redundant labels generated. However, as the graph size also increases the upper bound becomes less tight and so the algorithm begins to fail to discard redundant labels early. This behavior can be seen by noticing as the fraction of impeded edges increases for a fixed grid size, the computation time begins to grow at an exponential rate. Conversely, for a fixed fraction of impeded edges, the computation time grows more slowly as the grid size increases. 
\color{black}

\subsection{Class 2 instances}

For this set of instances, we wish to determine how the presence of unavoidable impeded edges affects the optimal solution cost. To do this, we use a fixed grid of size $3 \times 15$ with the same cost structure defined in \ref{PLA-baseline} for each instance. We then generated cuts to introduce unavoidable impeded edges for the convoy. The narrow grid structure was chosen as it is easier to produce cuts with fewer edges, which keeps the computational time reasonable as we increase the number of cuts. The service vehicle's starting position was chosen randomly for each instance. The number of cuts was varied form 1 to 5. We generated 50 random instances for each cut size.


\begin{table}[!ht]
\vspace*{0.1in}
\caption{Class 2 Results}
\label{class2}
    \begin{center}
        \begin{tabular}{c c c c c c}
            \toprule
            Cuts & $|K|/|E|$ & $OPT$ & $OPT/UB$ & $OPT/LB$ &$\sigma(OPT)$ \\ 
            \cmidrule(lr){1-6}
            1  &  0.05 & 191 & 0.89  & 1.05  & 6.19 \\
            2  &  0.10 & 196 & 0.82  & 1.08  & 6.48 \\
            3  &  0.14 & 201 & 0.77  & 1.10  & 6.94 \\
            4  &  0.18 & 205 & 0.71  & 1.12  & 6.90 \\
            5  &  0.22 & 207 & 0.68  & 1.14  & 7.56 \\
            \bottomrule
        \end{tabular}
    \end{center}
\end{table}

The computational results for this set of instances is shown in Table \ref{class2}. Each row in Table \ref{class2} represents an average value over the $50$ instances for a set number of cuts. Column $|K|/|E|$ indicates the fraction of impeded edges. An upper bound, $UB$, was computed by having the convoy take the shortest path from $p$ to $d$ without any assistance from the service vehicle. The gap between the cost of the optimal solution to the ASPP and $UB$ gives an indicator of the benefit of deploying the service vehicle.
Similarly, a lower bound, $LB$, was computed by treating all impeded edges as unimpeded and finding the shortest path for the convoy from $p$ to $d$ only using the unimpeded cost for all edges.
Columns $OPT/UB$ and $OPT/LB$ represents the ratio of optimal cost, $OPT$, against the upper bound and the lower bound costs, respectively. 
We see as the number of cuts increases, the $OPT/UB$ ratio decreases.
This shows the effectiveness/benefit of the service vehicle in an impeded environment, especially as the number of unavoidable impeded edges increases.
Column $\sigma(OPT)$ shows the standard deviation of the optimal cost. The increasing trend of $\sigma(OPT)$ indicates that the optimal solution cost is sensitive to the number of unavoidable impeded edges. \newline

\subsection{Class 3 instances}

For this set of instances, we examine the impact of the starting position of the service vehicle on the optimal solution. For each instance, we again use a $3 \times 15$ grid, but we impose a fixed cost structure for all instances. All edges were assigned costs $T^u_e = 10$ and $\tau^u_e = 1$ and the impeded edges were all assigned costs $T^i_e = 40$ and $\tau^i_e = 6$. The convoy's starting position and destination were fixed at diagonally opposite ends with positions $(0, 0)$ and $(14, 2)$ (see Figure \ref{cost_map}). This cost structure and choice of positions were chosen to encourage collaboration between the convoy and service vehicle. We generated 50 different instances by randomly generating 3 cuts for each instance. For a given instance, we compute the optimal solution cost for each of the 45 possible starting positions for the service vehicle. We then computed the average optimal cost across all 50 instances for each of the 45 service vehicle starting positions. 


\begin{figure}[!ht]
\begin{center}
\includegraphics[width=\linewidth]{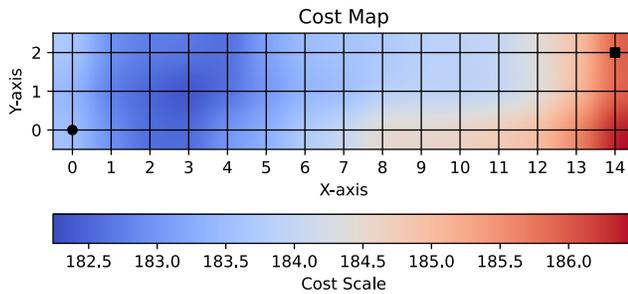}
\caption{Average cost map as a function of the starting position of the service vehicle.}
\label{cost_map} 
\end{center}
\end{figure}

Fig. \ref{cost_map} shows the average cost across $50$ instances for each starting position in the grid. The convoy starting position and destination are marked at $(0,0)$ and $(14,2)$.
Gaussian interpolation was used to generate a continues map. We observe that the minimum and maximum average costs are achieved at vertex $(3,1)$ and $(14,0)$ respectively. Fig. \ref{cost_map} shows the starting position of the service vehicle will have a noticeable impact on the optimal solution. The magnitude of the impact will depend on the numerical values of the costs for the edges. We also note the best starting position for the service vehicle in this set of instances does not coincide with the starting position of the convoy. This is to be expected. If the service vehicle starts some distance away from the convoy, it may be able to attend to more significant impeded edges before the convoy is able to reach them.

\section{Conclusion and Future Work} \label{Conclusion}
In this article, we considered an assisted path planning problem for a convoy and a service vehicle in an impeded environment. Although the convoy is capable of traversing impeded edges, the service vehicle can help the convoy to reach the destination with a lower total accumulated cost. A generalized permanent labeling algorithm ($GPLA$) was presented to solve the ASPP to optimality. Additional modifications were introduced to reduce the search space and the overall computational time as a result. The modified algorithm is referred as $GPLA^*$. 
Computational results showed the effectiveness of the modifications in $GLPA^*$ when compared to $GLPA$. Another computational study was done to show the computational performance of $GLPA^*$ as the problem size increased and as the number of impeded edges increased. It was found the number of impeded edges has a more significant impact on the computation time of $GLPA^*$ versus the overall size of the problem (i.e., the number of vertices). Two more computational studies were performed to study the impact of the presence of the number of unavoidable impeded edges and of the starting position of the service vehicle on the optimal solution to the ASPP. 


The future potential extensions of this problem may include limited communication range between the vehicles or the use of specialized service vehicles that can only repair some specific type of impeded edge. A non-deterministic version of ASPP can also be considered where a probability density function is assigned with the travel cost of an impeded edge.

\addtolength{\textheight}{-12cm}   

\bibliographystyle{ieeetr}
\bibliography{citations.bib}

\begin{thebibliography}{10}

\bibitem{wagner2015subdimensional}
G.~Wagner and H.~Choset, ``Subdimensional expansion for multirobot path
  planning,'' {\em Artificial intelligence}, vol.~219, pp.~1--24, 2015.

\bibitem{sharon2015conflict}
G.~Sharon, R.~Stern, A.~Felner, and N.~R. Sturtevant, ``Conflict-based search
  for optimal multi-agent pathfinding,'' {\em Artificial Intelligence},
  vol.~219, pp.~40--66, 2015.

\bibitem{de2013push}
B.~De~Wilde, A.~W. Ter~Mors, and C.~Witteveen, ``Push and rotate: cooperative
  multi-agent path planning,'' in {\em Proceedings of the 2013 international
  conference on Autonomous agents and multi-agent systems}, pp.~87--94, 2013.

\bibitem{standley2010finding}
T.~Standley, ``Finding optimal solutions to cooperative pathfinding problems,''
  in {\em Proceedings of the AAAI Conference on Artificial Intelligence},
  vol.~24, pp.~173--178, 2010.

\bibitem{erdmann1987multiple}
M.~Erdmann and T.~Lozano-Perez, ``On multiple moving objects,'' {\em
  Algorithmica}, vol.~2, no.~1, pp.~477--521, 1987.

\bibitem{silver2005cooperative}
D.~Silver, ``Cooperative pathfinding,'' in {\em Proceedings of the aaai
  conference on artificial intelligence and interactive digital entertainment},
  vol.~1, pp.~117--122, 2005.

\bibitem{li2016hybrid}
J.~Li, G.~Deng, C.~Luo, Q.~Lin, Q.~Yan, and Z.~Ming, ``A hybrid path planning
  method in unmanned air/ground vehicle (uav/ugv) cooperative systems,'' {\em
  IEEE Transactions on Vehicular Technology}, vol.~65, no.~12, pp.~9585--9596,
  2016.

\bibitem{parker_thesis}
L.~Parker, {\em Heterogeneous multi-robot cooperation}.
\newblock PhD thesis, MIT, 1994.

\bibitem{steels1990}
L.~Steels, ``Cooperation between distributed agents through
  self-organization,'' {\em IEEE International Workshop on Intelligent Robots
  and Systems, Towards a New Frontier of Applications}, pp.~8--14, 1990.

\bibitem{lynne1999}
L.~E. Parker, ``Cooperative robotics for multi-target observation,'' {\em
  Intelligent Automation \& Soft Computing}, vol.~5, no.~1, pp.~5--19, 1999.

\bibitem{minaeian2015vision}
S.~Minaeian, J.~Liu, and Y.-J. Son, ``Vision-based target detection and
  localization via a team of cooperative uav and ugvs,'' {\em IEEE Transactions
  on systems, man, and cybernetics: systems}, vol.~46, no.~7, pp.~1005--1016,
  2015.

\bibitem{garcia2017coordinated}
E.~Garcia and D.~Casbeer, ``Coordinated threat assignments and mission
  management of unmanned aerial vehicles,'' {\em Cooperative Control of
  Multi-Agent Systems: Theory and Applications}, p.~141, 2017.

\bibitem{chung2021offensive}
T.~H. Chung, ``Offensive swarm-enabled tactics (offset),'' DARPA, 2021.

\bibitem{donald_box_pushing}
B.~Donald, J.~Jennings, and D.~Rus, ``Analyzing teams of cooperating mobile
  robots,'' {\em Proceedings of the 1994 IEEE International Conference on
  Robotics and Automation}, vol.~3, pp.~1896--1903, 1994.

\bibitem{cooperative_Robots_Survey}
Y.~Cao, A.~Fukunaga, and A.~Kahng, ``Cooperative mobile robotics: Antecedents
  and directions,'' {\em Autonomous Robots}, pp.~7--27, 1997.

\bibitem{montez2021finding}
C.~Montez, S.~Rathinam, S.~Darbha, and D.~Casbeer, ``Finding shortest paths for
  a team of convoy and repair vehicles,'' in {\em AIAA Scitech 2021 Forum},
  p.~1769, 2021.

\bibitem{montez2021approximation}
C.~Montez, S.~Rathinam, S.~Darbha, D.~Casbeer, and S.~G. Manyam, ``An
  approximation algorithm for an assisted shortest path problem,'' in {\em 2021
  IEEE International Conference on Robotics and Automation (ICRA)},
  pp.~8024--8030, IEEE, 2021.

\bibitem{desrochers1988generalized}
M.~Desrochers and F.~Soumis, ``A generalized permanent labelling algorithm for
  the shortest path problem with time windows,'' {\em INFOR: Information
  Systems and Operational Research}, vol.~26, no.~3, pp.~191--212, 1988.

\bibitem{irnich2005shortest}
S.~Irnich and G.~Desaulniers, ``Shortest path problems with resource
  constraints,'' in {\em Column generation}, pp.~33--65, Springer, 2005.

\bibitem{hart1968formal}
P.~E. Hart, N.~J. Nilsson, and B.~Raphael, ``A formal basis for the heuristic
  determination of minimum cost paths,'' {\em IEEE transactions on Systems
  Science and Cybernetics}, vol.~4, no.~2, pp.~100--107, 1968.

\end{thebibliography}

\end{document}